\documentclass{article}
\usepackage{iclr2027_conference,times}
\iclrfinalcopy
\usepackage{amsmath,amssymb,amsthm}
\usepackage{graphicx,booktabs,array}
\usepackage{wrapfig}
\usepackage{xcolor}
\usepackage{hyperref}
\usepackage{url}
\hypersetup{pdfauthor={Chen Liang, Daoping Wu, Donghua Yang, Zhiyu Liang, Ying Zhang, Shuguang Xiong, Hongzhi Wang},pdftitle={DeltaTrace: Efficient Signed Attribution for Reasoning Language Models},colorlinks=true,linkcolor=blue!45!black,citecolor=blue!45!black,urlcolor=blue!45!black}

\title{DeltaTrace: Efficient Signed Attribution\\for Reasoning Language Models}
\author{
\textbf{Chen Liang$^{1,2}$ \quad Daoping Wu$^{2}$ \quad Donghua Yang$^{1}$ \quad Zhiyu Liang$^{1}$}\\
\textbf{Ying Zhang$^{2}$ \quad Shuguang Xiong$^{2}$ \quad Hongzhi Wang$^{1}$}\\[3pt]
{\normalfont $^{1}$Harbin Institute of Technology \quad $^{2}$Zhejiang Lab}\\[3pt]
{\normalfont\footnotesize\texttt{23B903050@stu.hit.edu.cn \quad daopingwu@outlook.com \quad yang.dh@hit.edu.cn}}\\
{\normalfont\footnotesize\texttt{\{zyliang,wangzh\}@hit.edu.cn \quad \{zhangy,xiongshuguang\}@zhejianglab.org}}
}
\newcommand{\dt}{\textsc{DeltaTrace}}
\newcommand{\D}{\Delta}
\newcommand{\R}{\mathbb{R}}
\newcommand{\ip}[2]{\left\langle #1,#2\right\rangle}
\newcommand{\sm}{\operatorname{softmax}}
\newcommand{\diag}{\operatorname{diag}}
\newcommand{\silu}{\operatorname{SiLU}}

\newtheorem{theorem}{Theorem}
\newtheorem{proposition}{Proposition}
\definecolor{support}{HTML}{176B87}
\definecolor{oppose}{HTML}{B25332}
\definecolor{ink}{HTML}{263442}

\begin{document}
\maketitle
\lhead{Preprint}
\suppressfloats[t]
\begin{abstract}

We study how words in the context support or oppose a language model's response and present \dt{}, a novel attribution framework that explains the effect of input tokens on a fixed model response. \dt{} measures how the response's log-probability changes when the context tokens being explained are masked. To distribute this difference among source tokens, it constructs operator-specific local rules for finite changes and composes them backward through the computational graph using effective secant slopes. Each token is assigned a signed attribution as the inner product of its embedding displacement and the corresponding backward-propagated slope. Local conservation makes these contributions sum to the measured score difference and preserves their meaning when aggregated over sentences or passages. \dt{} supports both attention and linear-attention models, with auxiliary memory growing linearly with sequence length at fixed model dimensions and chunk size. On the evaluated attention model, \dt{} achieves the lowest deletion-faithfulness scores (RISE and MAS) on most retrieval and reasoning tasks and the highest needle-in-a-haystack Recall, exceeding FlashTrace by 5.67-24.33 percentage points. On the evaluated hybrid attention and linear-attention model, \dt{} lowers RISE on 9 tasks and MAS on 12 relative to FlashTrace.
\end{abstract}

\section{Introduction}
\label{sec:introduction}
\begin{wrapfigure}[24]{r}{2.70in}
\vspace{-4.25\baselineskip}
\centering
\includegraphics[width=2.70in]{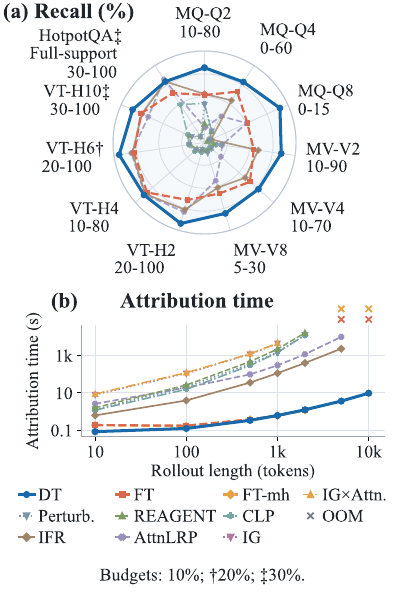}
\vspace{-\intextsep}
\end{wrapfigure}

As people increasingly use large language models (LLMs) to support decisions, they need to know what information the models rely on. Input attribution helps answer this question by tracing a response back to the words in its context. A reasoning response depends on both the information available and the computations that select and combine it \citep{ferrando2024information}. Figure~\ref{fig:overview}(d) shows a passage containing a playwright, William Shakespeare, and a composer, William Walton. The names supply candidate answers, and the phrase ``the author of the play'' specifies which role to use. Attribution should account for how both the candidate names and the requested role affect the response. It should also distinguish words that support the response from words that oppose it.

We ask how masking selected context tokens changes a language model's score for a fixed response. We compare the original prompt with a reference in which the selected tokens are replaced by EOS, while the response and its preceding tokens remain fixed. The resulting score difference is divided into signed contributions for the selected tokens.

\begin{figure}[t]
\centering
\input{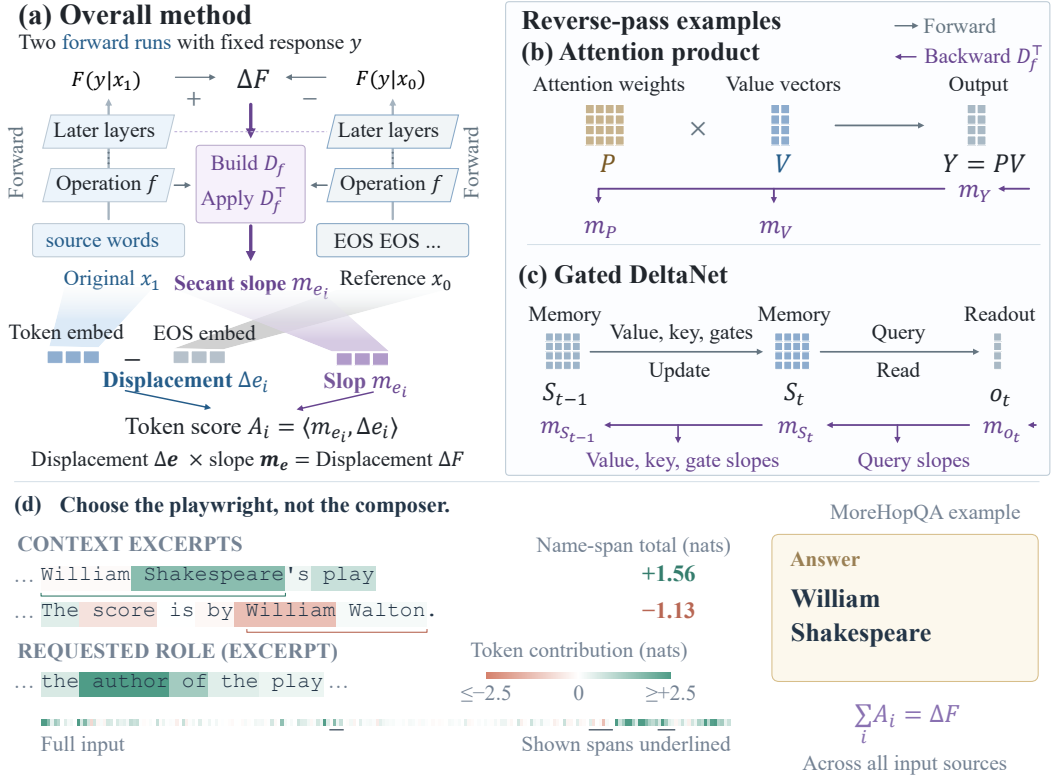}
\caption{\textbf{DeltaTrace assigns a response-score difference to input tokens.} (a) Two executions score the same fixed response $y$, with source tokens kept in the original input and masked with EOS in the reference. Each operation's formula and values from both runs determine local maps $D_f$ that express output changes as weighted sums of input changes. Applying their transposes in reverse yields $m$, the effective secant slopes from the indicated activations to the response log-probability $F$. At the inputs, multiplying these slopes by embedding displacements and summing within each token gives contributions whose total recovers the measured $\Delta F$. $\Delta$ denotes original minus reference. (b, c) Attention and linear attention (Gated DeltaNet) are examples within this reverse pass. Gray arrows show forward operations, and purple arrows return these slopes to their inputs. Other operations appear in Appendix~\ref{app:operators}. (d) Positive and negative scores identify support and opposition. Name annotations sum token scores, colors saturate at $\pm2.5$ nats, and ellipses omit text. The target is the full stored response plus EOS.}
\label{fig:overview}
\end{figure}

Attribution methods can explain prediction differences relative to a reference \citep{sundararajan2017axiomatic,shrikumar2017deeplift}. Relevance propagation and representation decomposition trace information through attention and recurrent states \citep{ali2022conservative,achtibat2024attnlrp,jafari2024mambalrp,modarressi2023decompx,pitorro2025latim}. FlashTrace follows source importance through multi-token responses \citep{pan2026flashtrace}. For the fixed-response question above, the target is a finite log-probability difference. Computing token contributions requires relating the response-score difference to activation differences throughout the model.

Existing attribution methods often propagate local derivatives or relevance. For the counterfactual above, this leaves a gap. Masking creates a finite score change, and several operands can change together inside one operation. In attention, both the weights and values change. In Gated DeltaNet, values, keys, queries, and gates jointly affect memory updates and reads \citep{yang2024gated}. The missing rule is one that expresses the observed output change in terms of these joint input changes before carrying it backward.

\textbf{\dt{}} traces the finite score difference $\D F$ back to its source tokens. It modifies backpropagation because derivatives evaluated at the original execution generally cannot reconstruct the finite response-score change caused by masking. From each operation's formula and the two executions, it constructs local finite-change rules and composes them in one reverse pass into effective secant slopes from input embeddings to the response score. Just as a secant slope converts horizontal displacement into vertical displacement, multiplying each embedding component's change by its effective slope gives its signed share of the score change. Summing these shares within a token gives its attribution; local conservation makes the total across all tokens equal the measured difference and preserves this interpretation when tokens are grouped into sentences or passages. Figure~\ref{fig:overview}(a) shows the complete method, while (b) and (c) illustrate its backward calculations for two operations.

\paragraph{Contributions.}
\begin{itemize}
\item \textbf{Signed evidence for the reasoning and answer together.} DT combines a source token's effects across response positions into one contribution to the full response's score difference. The signed scores identify support and opposition and can be summed to explain sentences or passages (Section~\ref{sec:reverse}).
\item \textbf{One procedure for attention and linear attention.} We derive backward calculations for attention products, normalization, and Gated DeltaNet as examples of the operations handled by \dt{}. These calculations account for simultaneous changes in their inputs, allowing the same attribution procedure to follow attention and recurrent memory (Sections~\ref{sec:attention}--\ref{sec:memory}).
\item \textbf{Auxiliary storage linear in sequence length.} FlashAttention-style tiles \citep{dao2022flashattention} and Flash Linear Attention (FLA) chunks \citep{yang2024fla} bound the workspace used for intermediate products. At fixed model dimensions and chunk size, longer inputs and responses require only linear growth in auxiliary storage (Section~\ref{sec:computation}).
\end{itemize}

On Qwen3-8B retrieval and reasoning tasks, \dt{} achieves the lowest deletion-faithfulness scores on most tasks and the highest evidence recovery across the six needle-in-a-haystack tasks (Section~\ref{sec:evaluation}). On Qwen3.5-9B, \dt{} lowers RISE on 9 of 13 tasks and MAS on 12 relative to FlashTrace.

\section{DeltaTrace}
\label{sec:method}

\subsection{The counterfactual question and attribution target}
We explain how masking selected prompt tokens changes the score of a fixed target span. The inputs are a trained language model with fixed parameters $\theta$, a prompt $x$ containing the context and question, and a stored response $y=(y_1,\ldots,y_L)$ of $L$ tokens. Let $\mathcal{T}$ be the positions of the target span within $y$. Both executions score these positions using the same preceding response tokens. The scalar target score is
\begin{equation}
F(x;y)=\sum_{t\in\mathcal{T}}\log p_\theta(y_t\mid x,y_{<t}),
\label{eq:target}
\end{equation}
where $p_\theta$ is the model's next-token probability and $y_{<t}$ is the response prefix before position $t$. A higher score means the model assigns greater probability to the target span given its preceding tokens. For full-response attribution, $y$ includes the reasoning, answer, and end-of-sequence (EOS) token, and $\mathcal{T}=\{1,\ldots,L\}$. The response tokens and target positions stay fixed, while their hidden representations can change as they read the different prompts.

The source positions are the context or question tokens chosen for explanation. Let $x_1$ be the original prompt. Its reference $x_0$ masks all source tokens together with EOS at the same positions (Figure~\ref{fig:overview}(a)), keeping all other prompt tokens fixed. Both executions process the full input sequence and score the same target span, giving
\[
\D F=F(x_1;y)-F(x_0;y).
\]
A positive $\D F$ means masking these tokens lowers the target score. The attribution task is to divide this difference among the source tokens. Each token receives a signed contribution $A_i$, and adding the contributions recovers $\D F$. Positive and negative scores indicate support and opposition under this reference (Figure~\ref{fig:overview}(d)).

Subscripts $0$ and $1$ denote reference and original activations, and $\D U=U_1-U_0$ denotes their difference. For source token $i$, $e_{i,0}$ is the EOS embedding and $e_{i,1}$ is its original embedding, so $\D e_i=e_{i,1}-e_{i,0}$. Tokens held fixed have zero embedding difference. The reverse pass converts these embedding differences into token contributions.

\subsection{From the score difference to token contributions}
\label{sec:reverse}
\dt{} replaces tangent slopes with secant slopes to trace the measured output change back to the inputs. For a scalar activation, the secant connects its reference and original points: its slope converts the horizontal input displacement into the vertical output displacement. A tangent slope gives only a first-order approximation to this finite change. Equal inputs use the local derivative.

The reverse pass carries $m$ for each activation entry, its effective secant slope to the final response score under the chosen allocation rules. Multiplying this slope by the entry's displacement gives its assigned contribution. Subscripts identify the associated activation. The pass starts with $m_F=1$, since the full score difference is to be allocated. Table~\ref{tab:reverse-comparison} contrasts tangent and secant propagation through the same activation.

\begin{table}[htbp]
\centering
\small
\caption{\textbf{The same activation, two backward calculations.} The activation has input $a$ and output $h=f(a)$. Here $d$ denotes an infinitesimal change and $\Delta$ the observed difference between the two executions.}
\label{tab:reverse-comparison}
\renewcommand{\arraystretch}{1.35}
\begin{tabular}{@{}lcc@{}}
\toprule
 & \shortstack{Ordinary backpropagation\\Tangent} & \shortstack{\dt{}\\Secant} \\
\midrule
Local slope & $f'(a_1)$ & $\dfrac{\D h}{\D a}=\dfrac{h_1-h_0}{a_1-a_0}$ \\
\addlinespace[3pt]
Output change & $dh=f'(a_1)\,da$ & $\D h=\dfrac{\D h}{\D a}\,\D a$ \\
\addlinespace[3pt]
Backward calculation & $\dfrac{\partial F}{\partial a}=f'(a_1)\dfrac{\partial F}{\partial h}$ & $m_a=\dfrac{\D h}{\D a}\,m_h$ \\
\bottomrule
\end{tabular}
\end{table}

For array operations, we use the operation's formula and both executions to express output changes as weighted sums of input changes. These local slopes form the matrix $D_f$. Linear operations use their existing weights. For a product, splitting the term involving both input changes equally between them gives equal weight to the two possible change orders. This produces the mean attention weights and mean values in Section~\ref{sec:attention}.

The same local map relates observed differences and propagates $m$ backward through its transpose,
\begin{equation}
\D h=D_f\D a,
\qquad m_a=D_f^\top m_h.
\label{eq:reverse-contribution}
\end{equation}
The transpose multiplies each output's effective slope by the corresponding local slopes and sums the products for each input. This expresses the same contribution in terms of input changes. The local slopes replace derivatives, and the propagated slopes replace gradients. The two products have the forms of a Jacobian--vector product (JVP) and a vector--Jacobian product (VJP), respectively.\footnote{PyTorch documentation for \href{https://docs.pytorch.org/docs/2.14/generated/torch.func.jvp.html}{\texttt{torch.func.jvp}} and \href{https://docs.pytorch.org/docs/2.14/generated/torch.func.vjp.html}{\texttt{torch.func.vjp}}.}

At the input embeddings, the propagated $m_{e_i}$ give the effective secant slopes for each token. Summing their products with the corresponding embedding displacements yields the token contributions,
\begin{equation}
A_i=\ip{m_{e_i}}{\D e_i},
\qquad \sum_i A_i=\D F.
\label{eq:attribution}
\end{equation}
Every backward calculation preserves the contribution being divided, so the token scores add to the measured response-score difference. Positive and negative scores identify support and opposition across the chosen response span. They use natural-log units (nats) and can be summed into name, sentence, or passage contributions. A fixed token's embedding does not change, so its contribution is zero. Appendix~\ref{app:conservation} gives the full derivation.

We next apply these rules to attention and Gated DeltaNet (Figure~\ref{fig:overview}(b, c)). Appendix~\ref{app:operators} gives the formulas for linear layers, additions, activation functions, and normalizations in the same reverse pass.

\subsection{Attention example}
\label{sec:attention}
For attention, $P_{ij}$ is the weight with which position $i$ reads source $j$, and the rows of $V$ contain source value vectors. The output $Y=PV$ is their weighted sum, with $P$ obtained by softmax normalization of query--key similarity scores.

We use the two-run means so that the weighted input changes sum exactly to the observed output change. Let $\D P$, $\D V$, and $\D Y$ denote the differences in weights, values, and output between the two executions. Splitting the interaction term $\D P\,\D V$ equally between the inputs gives equal weight to the two possible change orders:
\begin{equation}
\begin{aligned}
\D Y
&=P_0\D V+\D P\,V_0+\underbrace{\D P\,\D V}_{\text{split equally}}\\
&=\underbrace{\left(P_0+\tfrac12\D P\right)}_{\bar P}\D V
+\D P\underbrace{\left(V_0+\tfrac12\D V\right)}_{\bar V}.
\end{aligned}
\label{eq:pv}
\end{equation}
Here $\bar P$ and $\bar V$ are the means of the original and reference weights and values. They provide the two sets of local slopes in $D_f$ (Equation~\ref{eq:reverse-contribution}): the mean weights multiply the value changes, and the mean values multiply the weight changes.

The incoming slopes $m_Y$ weight the two terms in the second line of Equation~\ref{eq:pv}, dividing the output's assigned contribution between value and weight changes. The backward calculation uses the same two products as ordinary backpropagation, with means replacing the original activations,
\begin{equation}
\begin{aligned}
\text{Backpropagation}\quad &
\frac{\partial F}{\partial V}=P_1^\top\frac{\partial F}{\partial Y},
&\quad \frac{\partial F}{\partial P}=\frac{\partial F}{\partial Y}V_1^\top,\\
\text{DeltaTrace}\quad &
m_V=\bar P^\top m_Y,
&\quad m_P=m_Y\bar V^\top.
\end{aligned}
\label{eq:pv-reverse}
\end{equation}
The returned slopes continue through preceding operations to the input embeddings.

Softmax makes each attention row sum to one, so increasing one source's weight reduces the total assigned to the others. The backward calculation for softmax carries this competition through a weighted row average subtracted from the incoming slopes. Its weights depend on the probabilities in both executions. The slopes then pass through the query--key scores to their input representations. The same construction handles the vocabulary normalization used to score each response token. Appendix~\ref{app:softmax} gives the formula and its derivation.

\subsection{Gated DeltaNet example}
\label{sec:memory}
Gated DeltaNet implements linear attention through recurrent memory that stores key--value associations across token positions \citep{yang2024gated}. Figure~\ref{fig:overview}(c) shows one step from the previous memory $S_{t-1}$ to the updated memory $S_t$ and its readout $o_t$. A retention gate controls how much earlier memory is kept, and a write gate controls a correction toward the current value at its key. A query reads from $S_t$ to produce $o_t$. Normalization, an output gate, and a linear projection then produce the layer output. Appendix~\ref{app:memory} gives the state equations.

Information written by an earlier token can affect a later readout, so its contribution must pass backward through the intervening memory updates. The readout slopes $m_{o_t}$ pass backward to the memory and query. At $S_t$, the contribution from this readout joins those from later memory states. The resulting $m_{S_t}$ passes through the update to $S_{t-1}$, the value, the key, and the retention and write gates (Figure~\ref{fig:overview}(c)). The returned slopes continue through preceding operations to the input embeddings, where Equation~\ref{eq:attribution} gives token scores.

Within each GDN layer, the original and reference activations define two decompositions of the same readout changes over the complete memory sequence. Both backward calculations receive the same readout slopes. Their returned slopes for each value, key, query, and gate are averaged with equal weight before propagation continues to earlier layers. This averaging stays within the model's single reverse pass. The separate output-gating product also splits contributions between its two changing inputs using their two-run means. Appendix~\ref{app:memory} gives these calculations.

\section{Efficient Computation}
\label{sec:computation}
Storing dense attention matrices requires space for all token pairs. \dt{} computes local transpose products in small blocks and accumulates their results directly into the input slopes. The attention and memory rules in Sections~\ref{sec:attention}--\ref{sec:memory} support tiled and chunked implementations.

\paragraph{Attention in tiles.}
A tile is a small block of query and source positions. Following FlashAttention \citep{dao2022flashattention}, \dt{} reconstructs the two executions' attention probabilities for one tile, uses them to update query, key, and value slopes, and reuses the workspace for the next tile. The softmax rule needs a weighted average over the whole row. One sweep collects those row statistics. Subsequent sweeps use them when propagating through each tile. For $N$ input-plus-response positions and head dimension $d$, the stored vectors and row statistics occupy $O(Nd)$ auxiliary space per head. Dense attention requires $O(N^2d)$ arithmetic.

\paragraph{Memory in chunks.}
A chunk is a consecutive group of tokens. The memory state at a chunk boundary summarizes what earlier tokens passed forward. In reverse, the slopes for that state pass back to the preceding chunk. Within each chunk, matrix products propagate the slopes through its reads, writes, and gates. This organization uses Flash Linear Attention (FLA) \citep{yang2024fla}. At fixed state dimensions and chunk size, the stored token slopes and boundary states grow linearly with sequence length.

The complete attribution call includes two forward executions and one reverse pass. It recomputes intermediate activations one layer at a time during the reverse pass, so each layer's workspace can be reused for the next. Appendix~\ref{app:computation} details the storage accounting and accelerated operations. Section~\ref{sec:evaluation} measures the complete call's time and memory.

\section{Evaluation}
\label{sec:evaluation}
We evaluate evidence recovery, deletion faithfulness, and attribution cost.

\paragraph{Tasks and methods.}
We first evaluate Qwen3-8B on the 1,243 examples released with FlashTrace \citep{pan2026flashtrace}. These comprise ten RULER tasks with 100 examples each, HotpotQA with 48, MATH with 100, and MoreHopQA with 95.\footnote{\url{https://github.com/wbopan/flashtrace/releases/tag/table1-data-v1}} Needle-in-a-haystack (NIAH) tasks retrieve information among distractors. Multi-query tasks (MQ-Q2/Q4/Q8) ask for 2, 4, or 8 keys, while multi-value tasks (MV-V2/V4/V8) ask for 2, 4, or 8 values of one key. Variable Tracking (VT) follows variable assignments. Here H is the number of binding hops and C the number of parallel chains. HotpotQA and MoreHopQA require combining facts, while MATH contains mathematical reasoning problems. The full stored reasoning and answer plus EOS is scored for faithfulness. We compare \dt{} (DT) with FlashTrace (FT), Perturbation, REAGENT, Context Length Probing (CLP), Information Flow Routes (IFR), and Attention-aware Layer-wise Relevance Propagation (AttnLRP). Table~\ref{tab:baseline-methods} in Appendix~\ref{app:evaluation} summarizes their principles and DT's advantages. DT and FT use the same evaluation examples. VT/HotpotQA recovery compares all seven methods.

\paragraph{Evidence recovery.}
Recall is the fraction of annotated evidence recovered among the highest-ranked sources. NIAH evidence consists of the inserted key--value sentences, and VT evidence consists of the assignment statements along the target variable chain \citep{pan2026flashtrace}. A token budget limits how much input the ranking may select. For NIAH, DT attributes the full response plus EOS. We select the top 10\% of eligible input tokens ranked by positive-part attribution scores. Eligible tokens are the source positions marked for evaluation in the released examples. VT measures token Recall in the passage body, with budgets of 10\%, 10\%, 20\%, and 30\% for H2, H4, H6, and H10, and attributes the answer tokens. HotpotQA reports Full-support Recall@30\%. An example scores one if every official supporting fact is selected, and zero otherwise. It attributes the full response, ranks passage sentences by the sum of their positive token contributions, and selects the longest ranking prefix within 30\% of all body tokens. FT uses three recursive tracing steps (K3) for recovery. Table~\ref{tab:recovery} reports these values.

\paragraph{Deletion faithfulness.}
The released deletion test masks high-ranked source tokens with EOS in 20 steps, rescoring the same full response after each step. The score is normalized between the full-input and fully-masked endpoints. A good ranking causes an early score drop because it masks evidence the model uses. The RISE deletion score summarizes the area under this curve. Magnitude Aligned Scoring (MAS) also compares the remaining response score with the remaining share of positive attribution \citep{pan2026flashtrace}. Lower values are better for both metrics. DT uses signed scores for RISE. For MAS, it sets negative scores to zero, ranking sources by their positive support for the response. Appendix~\ref{app:evaluation} specifies the ranking and normalization.

\begin{table}[t]
\centering\small
\caption{RISE $\downarrow$ on the complete released Qwen3-8B tasks. FT is FlashTrace \citep{pan2026flashtrace} and DT is \dt{}. DT uses signed ranking. FT uses the same evaluation examples as DT. $n$ is the released task size. Bold marks the lowest observed mean across all seven methods.}
\label{tab:full}
\setlength{\tabcolsep}{3.0pt}
\begin{tabular}{@{}lr rrrrrrr@{}}
\toprule
Task & $n$ & Perturbation & REAGENT & CLP & IFR & AttnLRP & FT & DT \\
\midrule
MQ-Q2 & 100 & 0.0945 & 0.1174 & 0.0977 & 0.0747 & 0.1955 & \textbf{0.0682} & 0.0723 \\
MQ-Q4 & 100 & 0.2389 & 0.2599 & 0.2534 & 0.1147 & 0.2628 & \textbf{0.1134} & 0.1448 \\
MQ-Q8 & 100 & 0.4986 & 0.4865 & 0.5096 & 0.3709 & 0.3766 & \textbf{0.3519} & 0.4286 \\
MV-V2 & 100 & 0.1343 & 0.1876 & 0.1559 & 0.0691 & 0.1403 & 0.0686 & \textbf{0.0597} \\
MV-V4 & 100 & 0.1862 & 0.2107 & 0.2166 & 0.0726 & 0.1934 & 0.0704 & \textbf{0.0587} \\
MV-V8 & 100 & 0.3514 & 0.3685 & 0.3929 & 0.2051 & 0.2852 & 0.1828 & \textbf{0.1159} \\
\midrule
VT-H2-C3 & 100 & 0.3844 & 0.4378 & 0.3742 & 0.1609 & 0.3191 & 0.1317 & \textbf{0.1003} \\
VT-H4-C1 & 100 & 0.3538 & 0.3971 & 0.3277 & 0.1018 & 0.3239 & 0.1102 & \textbf{0.1013} \\
VT-H6-C1 & 100 & 0.4584 & 0.4860 & 0.4225 & 0.1245 & 0.3378 & 0.1224 & \textbf{0.1109} \\
VT-H10-C1 & 100 & 0.4663 & 0.4953 & 0.4510 & 0.1526 & 0.3573 & 0.1427 & \textbf{0.1403} \\
\midrule
HotpotQA & 48 & 0.1333 & 0.1445 & 0.1013 & 0.0745 & 0.1548 & 0.0331 & \textbf{0.0275} \\
MATH & 100 & 0.3802 & 0.3878 & 0.3616 & 0.3537 & 0.3684 & 0.3487 & \textbf{0.2321} \\
MoreHopQA & 95 & 0.2489 & 0.2652 & 0.2487 & 0.1455 & 0.2857 & 0.1281 & \textbf{0.0703} \\
\bottomrule
\end{tabular}
\end{table}

\begin{table}[t]
\centering\small
\caption{MAS $\downarrow$ on the same released tasks and seven methods as Table~\ref{tab:full}. DT uses its positive source contributions. Bold marks the lowest observed mean.}
\label{tab:mas}
\setlength{\tabcolsep}{3.0pt}
\begin{tabular}{@{}lr rrrrrrr@{}}
\toprule
Task & $n$ & Perturbation & REAGENT & CLP & IFR & AttnLRP & FT & DT \\
\midrule
MQ-Q2 & 100 & 0.1442 & 0.1966 & 0.1663 & 0.1397 & 0.3261 & 0.1628 & \textbf{0.1176} \\
MQ-Q4 & 100 & 0.3270 & 0.3574 & 0.3203 & \textbf{0.1770} & 0.4509 & 0.1937 & 0.2011 \\
MQ-Q8 & 100 & 0.7093 & 0.6942 & 0.6566 & 0.4600 & 0.6020 & \textbf{0.4271} & 0.6047 \\
MV-V2 & 100 & 0.1870 & 0.2756 & 0.2159 & 0.1343 & 0.2290 & 0.1569 & \textbf{0.1046} \\
MV-V4 & 100 & 0.2436 & 0.2905 & 0.2804 & 0.1416 & 0.3248 & 0.1605 & \textbf{0.1058} \\
MV-V8 & 100 & 0.4580 & 0.4938 & 0.5114 & 0.2746 & 0.4752 & 0.2477 & \textbf{0.1717} \\
\midrule
VT-H2-C3 & 100 & 0.5513 & 0.6680 & 0.4900 & 0.2309 & 0.5214 & 0.1873 & \textbf{0.1486} \\
VT-H4-C1 & 100 & 0.5168 & 0.6031 & 0.4199 & 0.1479 & 0.5475 & 0.1656 & \textbf{0.1406} \\
VT-H6-C1 & 100 & 0.6837 & 0.7453 & 0.5648 & 0.1727 & 0.5723 & 0.1787 & \textbf{0.1537} \\
VT-H10-C1 & 100 & 0.7013 & 0.7405 & 0.5967 & 0.2011 & 0.5915 & 0.1983 & \textbf{0.1842} \\
\midrule
HotpotQA & 48 & 0.2203 & 0.2351 & 0.1895 & 0.1655 & 0.2493 & 0.1276 & \textbf{0.1002} \\
MATH & 100 & 0.5202 & 0.5358 & 0.4907 & 0.4897 & 0.5244 & 0.4461 & \textbf{0.3323} \\
MoreHopQA & 95 & 0.3470 & 0.3679 & 0.3478 & 0.2277 & 0.4579 & 0.2049 & \textbf{0.1546} \\
\bottomrule
\end{tabular}
\end{table}

\begin{table}[t]
\centering\footnotesize
\caption{Evidence recovery ($\%$, higher is better) at the listed token budgets. NIAH/VT report token Recall. HotpotQA reports Full-support Recall@30\%, the fraction of examples for which every official supporting fact is selected. Sentence scores sum positive token contributions. All body tokens are included in the budget calculation. VT macro averages its four rows. FT uses K3 for recovery. Bold marks the highest observed mean.}
\label{tab:recovery}
\setlength{\tabcolsep}{2.2pt}
\begin{tabular}{@{}lrr rrrrrrr@{}}
\toprule
Task & $n$ & Budget & Perturbation & REAGENT & CLP & IFR & AttnLRP & FT & DT \\
\midrule
MQ-Q2 & 100 & 10\% & 39.05 & 24.35 & 39.85 & 47.10 & 21.49 & 46.32 & \textbf{67.24} \\
MQ-Q4 & 100 & 10\% & 9.04 & 8.49 & 8.58 & 32.84 & 20.42 & 39.58 & \textbf{47.16} \\
MQ-Q8 & 100 & 10\% & 0.97 & 0.47 & 0.84 & 1.19 & 7.63 & 7.72 & \textbf{13.59} \\
MV-V2 & 100 & 10\% & 25.48 & 18.02 & 20.67 & 57.50 & 25.44 & 53.46 & \textbf{77.79} \\
MV-V4 & 100 & 10\% & 16.05 & 15.60 & 14.61 & 45.24 & 24.25 & 49.37 & \textbf{56.79} \\
MV-V8 & 100 & 10\% & 8.02 & 7.37 & 7.29 & 17.87 & 15.88 & 19.52 & \textbf{25.19} \\
\midrule
VT-H2-C3 & 100 & 10\% & 26.69 & 26.66 & 29.10 & 81.16 & 83.37 & 72.41 & \textbf{93.81} \\
VT-H4-C1 & 100 & 10\% & 18.86 & 17.58 & 20.53 & 70.09 & 70.54 & 68.32 & \textbf{71.31} \\
VT-H6-C1 & 100 & 20\% & 33.57 & 32.47 & 33.85 & 83.74 & 84.57 & 82.15 & \textbf{95.37} \\
VT-H10-C1 & 100 & 30\% & 40.70 & 42.48 & 44.04 & 82.66 & 77.33 & 83.41 & \textbf{90.51} \\
VT macro & 400 & mixed & 29.96 & 29.80 & 31.88 & 79.41 & 78.95 & 76.57 & \textbf{87.75} \\
\midrule
HotpotQA & 48 & 30\% & 35.42 & 35.42 & 64.58 & 85.42 & \textbf{87.50} & 75.00 & 85.42 \\
\bottomrule
\end{tabular}
\end{table}

\paragraph{Attribution quality.}
On Qwen3-8B, DT has the lowest RISE on 10 of 13 tasks and the lowest MAS on 11 (Tables~\ref{tab:full} and~\ref{tab:mas}). The largest absolute RISE reduction against FT is on MATH, from 0.3487 to 0.2321. DT achieves the highest Recall on all six NIAH tasks and all four VT tasks. At the stated VT budgets, its macro Recall is 87.75\%, compared with 76.57\% for FT K3. On HotpotQA, DT reaches 85.42\% Full-support Recall, compared with 75.00\% for FT K3 (Table~\ref{tab:recovery}).

\begin{table}[t]
\centering\small
\caption{Attribution quality on Qwen3.5-9B. All tasks use the full set of released examples. Recovery is token Recall for NIAH/VT and Full-support Recall@30\% for HotpotQA, as defined in Table~\ref{tab:recovery}. FT denotes FlashTrace, with K3 for Recall and K1 for RISE/MAS. DT uses signed RISE and positive-part MAS. Bold marks the better mean. Dashes indicate undefined recovery.}
\label{tab:qwen35-quality}
\setlength{\tabcolsep}{4.0pt}
\begin{tabular}{@{}lrr rr rr rr@{}}
\toprule
Task & $n$ & Budget & \multicolumn{2}{c}{Recall (\%) $\uparrow$} & \multicolumn{2}{c}{RISE $\downarrow$} & \multicolumn{2}{c}{MAS $\downarrow$} \\
\cmidrule(lr){4-5}\cmidrule(lr){6-7}\cmidrule(l){8-9}
 & & & FT & DT & FT & DT & FT & DT \\
\midrule
MQ-Q2 & 100 & 10\% & 52.92 & \textbf{71.94} & \textbf{0.0982} & 0.0996 & 0.2549 & \textbf{0.1306} \\
MQ-Q4 & 100 & 10\% & 38.56 & \textbf{54.18} & \textbf{0.1414} & 0.1579 & 0.2651 & \textbf{0.2242} \\
MQ-Q8 & 100 & 10\% & 10.72 & \textbf{14.63} & \textbf{0.3704} & 0.4028 & \textbf{0.4390} & 0.5891 \\
MV-V2 & 100 & 10\% & 60.51 & \textbf{83.53} & 0.0974 & \textbf{0.0804} & 0.2467 & \textbf{0.1080} \\
MV-V4 & 100 & 10\% & 46.19 & \textbf{61.42} & 0.1284 & \textbf{0.1130} & 0.2538 & \textbf{0.1510} \\
MV-V8 & 100 & 10\% & 22.27 & \textbf{26.36} & 0.2508 & \textbf{0.2206} & 0.3264 & \textbf{0.2992} \\
\midrule
MATH & 100 & --- & --- & --- & 0.3777 & \textbf{0.2638} & 0.4837 & \textbf{0.3368} \\
MoreHopQA & 95 & --- & --- & --- & 0.2093 & \textbf{0.1429} & 0.2894 & \textbf{0.2001} \\
\midrule
VT H2-C3 & 100 & 10\% & 67.80 & \textbf{94.20} & 0.1514 & \textbf{0.1095} & 0.2224 & \textbf{0.1295} \\
VT H4-C1 & 100 & 10\% & 69.07 & \textbf{71.89} & 0.1161 & \textbf{0.1013} & 0.1888 & \textbf{0.1343} \\
VT H6-C1 & 100 & 20\% & 83.06 & \textbf{97.35} & 0.1300 & \textbf{0.1154} & 0.2051 & \textbf{0.1504} \\
VT H10-C1 & 100 & 30\% & 81.90 & \textbf{94.23} & \textbf{0.1537} & 0.1561 & 0.2307 & \textbf{0.2040} \\
HotpotQA & 48 & 30\% & 58.33 & \textbf{95.83} & 0.2154 & \textbf{0.1470} & 0.3125 & \textbf{0.2367} \\
\bottomrule
\end{tabular}
\end{table}

\paragraph{Qwen3.5-9B results.}
We also evaluate DT and FT on the hybrid Qwen3.5-9B model using 13 tasks comprising 1,243 examples, including all 400 VT and 48 HotpotQA examples (Table~\ref{tab:qwen35-quality}). DT has lower RISE on 9 tasks and lower MAS on 12. Recall is higher on all six NIAH and all four VT tasks. On HotpotQA, DT reaches 95.83\% Full-support Recall, compared with 58.33\% for FT K3, and lowers both RISE and MAS.

Across both models, DT recovers more annotated MQ evidence within the same 10\% token budget, with higher Recall for 2, 4, and 8 queried keys.

\paragraph{Attribution time and memory.}
The first-page timing panel compares DT and one-hop FT as response length increases on Qwen3-8B, using one MetaX C550. Crosses mark out-of-memory runs. In the 16-example Qwen3-8B comparison, DT has comparable attribution time to one-hop FT (11.40 versus 12.16 seconds, summing per-example medians) with lower peak allocated memory (18.76 versus 21.49 GB). Times cover complete attribution calls. Appendix~\ref{app:cost} gives the timing protocol.

\paragraph{Reading the token scores.}
The signed scores distinguish support from opposition in the Qwen3.5-9B playwright example (Figure~\ref{fig:overview}(d)). Figures~\ref{fig:retrieval} and~\ref{fig:multihop} show retrieval and multi-hop examples with each model's own scores. Teal marks positive contributions and coral negative contributions to the full response score. The retrieval example links source keys to their numbers and the query naming them. The multi-hop example follows a voice actor to a film and its director, whose surname determines the answer. Each deletion curve compares DT with one-hop FT on the same model and response.

\section{Related Work}

\paragraph{Reference-based attribution.}
Integrated Gradients integrates input sensitivities between a reference and the input, while DeepLIFT propagates activation differences into source contributions \citep{sundararajan2017axiomatic,shrikumar2017deeplift}. Conservative Propagation and AttnLRP develop relevance rules for transformers \citep{ali2022conservative,achtibat2024attnlrp}. DT constructs local slopes from each operation's formula and two executions, then composes them in a single reverse pass to attribute the response's log-probability difference to input tokens.

\paragraph{Tracing source information.}
ALTI aggregates token interactions across layers, and DecompX propagates decomposed representations to the prediction head \citep{ferrando2022alti,modarressi2023decompx}. Information Flow Routes traces influential computations, while FlashTrace aggregates response spans and recursively follows their source importance \citep{ferrando2024information,pan2026flashtrace}. DT sums the target response tokens' log-probabilities into one scalar score and attributes its difference from the reference to input tokens.

\paragraph{Recurrent models.}
MambaLRP conserves relevance through selective state-space computation, and LaTIM decomposes token interactions in Mamba models \citep{jafari2024mambalrp,pitorro2025latim}. The Gated DeltaNet example applies DT's backward calculation to retained states, write corrections, and query reads \citep{yang2024gated}. The same reverse pass traverses attention and memory layers.

\section{Conclusion}
DT assigns signed token contributions to the change in a full response's log-probability between reference and original inputs. Its reverse pass replaces local derivatives with finite-change local slopes and composes them into effective secant slopes. Their products with embedding displacements give token contributions that sum to the response-score difference. Tiled attention and chunked memory make the auxiliary storage linear in sequence length at fixed model dimensions and chunk size. On the evaluated attention and hybrid models, DT improves evidence recovery and deletion faithfulness over FlashTrace on most tasks. The same counterfactual view has an action-level analogue: under a fixed continuation policy, averaging the return difference over a reference action recovers the advantage $Q^\pi(h,a)-V^\pi(h)$, linking finite counterfactual credit to the standard policy-gradient signal (Appendix~\ref{app:credit}).

\section*{AI Use Statement}
AI tools assisted with method development, derivations, coding, experiments, literature review, and manuscript preparation.

\section*{Ethics Statement}
The experiments use publicly released benchmark examples and do not involve human subjects or private personal data. The work presents a research tool for analyzing model behavior, and we have no additional conflicts of interest, sponsorship, or legal or ethical concerns to disclose.

\section*{Reproducibility Statement}
The paper and appendix specify the attribution procedure, finite propagation rules, attention and Gated DeltaNet cases, evaluation metrics, and measurement protocols. The supplementary code package provides the implementation, frozen result tables, figure-generation scripts, environment records, and reproduction instructions for the reported results.

\begingroup
\raggedright
\bibliography{references}
\endgroup
\bibliographystyle{iclr2027_conference}
\clearpage
\appendix
\section{Formal Rules and Derivations}
\label{app:proofs}

\subsection{Finite propagation and contribution conservation}
\label{app:conservation}
For an operation $v=f(u_1,\ldots,u_k)$, \dt{} constructs linear maps $D_j$ from the two executions such that
\begin{equation}
\D v=\sum_{j=1}^k D_j\D u_j.
\label{eq:local}
\end{equation}
The incoming slopes $m_v$ propagate to input $j$ as $D_j^\top m_v$. Slopes from multiple consumers add at their shared input. Starting with slope $1$ at the scalar target, the reverse pass produces the source slopes in Equation~\ref{eq:attribution}. For a scalar nonlinearity $f$, the local slope is the divided difference $(f(u_1)-f(u_0))/(u_1-u_0)$, with its derivative as the coincident-input limit. Linear transformations retain their weights, and addition sends the slope to each summand.

\begin{theorem}[Contribution conservation]
For an acyclic computation over real-valued operations satisfying Equation~\ref{eq:local}, reverse propagation from a scalar target yields the conservation identity in Equation~\ref{eq:attribution}.
\end{theorem}
\begin{proof}
An acyclic computation can be expanded into elementary operations with explicit branches. At an operation $v=f(u_1,\ldots,u_k)$, Equation~\ref{eq:local} gives
\begin{equation}
\ip{m_v}{\D v}=\sum_j\ip{D_j^\top m_v}{\D u_j}.
\end{equation}
Replacing a node's contribution by the contributions of its parents therefore preserves the total on a cut through the graph. At shared parents, slopes from different consumers sum. Repeated substitution reaches the input cut and yields Equation~\ref{eq:attribution}, since the slope of the scalar target is $1$. Constant inputs have zero difference and contribute zero to this sum.

\end{proof}

\subsection{Finite attention normalization}
\label{app:softmax}
For an attention row $p_b=\sm(z_b)$, $b\in\{0,1\}$, the finite rule first converts changes in log-probabilities into changes in probabilities. The slope of this conversion is the logarithmic mean $\ell_i$. We define it and its sum as
\begin{equation}
\ell_i=\frac{p_{1,i}-p_{0,i}}{\log p_{1,i}-\log p_{0,i}},\qquad
\tau=\sum_i\ell_i,
\label{eq:logmean}
\end{equation}
with $\ell_i=p_{0,i}$ at coincident probabilities. Dividing by $\tau$ gives weights that sum to one for the row averages below. All quantities are evaluated on the common unmasked support.
\begin{proposition}[Finite attention normalization]
For finite logits on a common support,
\begin{equation}
\D p=\left(\diag(\ell)-\frac{\ell\ell^\top}{\tau}\right)\D z.
\label{eq:softmax}
\end{equation}
\end{proposition}
\begin{proof}
The common support contains positive probabilities. For endpoint $b$, $\log p_b=z_b-a_b\mathbf 1$, where $a_b=\log\sum_i\exp(z_{b,i})$. The definition of the logarithmic mean gives
\begin{equation}
\D p=\ell\odot(\D z-\D a\,\mathbf 1).
\end{equation}
Probability normalization implies $\mathbf 1^\top\D p=0$, so
\begin{equation}
\D a=\frac{\ell^\top\D z}{\tau}.
\end{equation}
Substitution proves Equation~\ref{eq:softmax}. For a target coordinate $y$,
\begin{equation}
\D\log p_y=\left(e_y-\frac{\ell}{\tau}\right)^\top\D z.
\end{equation}
The coincident-probability definition follows from the continuous limit of the logarithmic mean.

\end{proof}
The operator is symmetric. For incoming slopes $m_p$ at probabilities $p$, the slope at query--key score $z_i$ is
\begin{equation}
m_{z,i}=\ell_i(m_{p,i}-\bar m_p),\qquad
\bar m_p=\frac{\sum_j\ell_j m_{p,j}}{\sum_j\ell_j}.
\label{eq:softmax-vjp}
\end{equation}
The slopes for a log-softmax target are $m_z=e_y-\ell/\tau$, where $e_y$ is the one-hot vector selecting the target vocabulary coordinate.

\subsection{Finite operator rules}
\label{app:operators}
For the attention example with output $Y=PV$, Equation~\ref{eq:pv} gives the slopes $m_V=\bar P^\top m_Y$ and $m_P=m_Y\bar V^\top$ for an incoming slope matrix $m_Y$.

Query--key scores use the symmetric product identity,
\begin{equation}
\D(QK^\top)=\D Q\,\bar K^\top+\bar Q\,\D K^\top,\qquad
\bar Q=(Q_0+Q_1)/2,
\label{eq:qk}
\end{equation}
with $\bar K$ defined in the same way. Table~\ref{tab:rules} lists the finite decompositions by operator. Normalization, rotary position transformations, and grouped heads follow their respective model definitions.

\begin{table}[htbp]
\centering
\small
\caption{Finite propagation rules by operator. Bars denote endpoint averages. The gated-memory recurrence is given in Appendix~\ref{app:memory}.}
\label{tab:rules}
\begin{tabular}{@{}p{0.27\linewidth}p{0.65\linewidth}@{}}
\toprule
Operation & Change decomposition \\
\midrule
Linear and residual & $\D(Wx+b)=W\D x$; $\D(a+b)=\D a+\D b$ \\
Attention values & $\D(PV)=\bar P\D V+\D P\bar V$ \\
Attention scores & $\D(QK^\top)=\D Q\bar K^\top+\bar Q\D K^\top$ \\
SwiGLU & $\D(us)=\bar s\odot\D u+\bar u\odot\D s$, $s=\silu(g)$ \\
GDN output gating & $\D(n\odot s)=\bar s\odot\D n+\bar n\odot\D s$ \\
RMS normalization & Analytic finite rule for $wx/\sqrt{\operatorname{mean}(x^2)+\epsilon}$ \\
\bottomrule
\end{tabular}
\end{table}

\subsection{Backward calculation for the Gated DeltaNet example}
\label{app:memory}
With column vectors and a key-by-value state $S_t\in\R^{d_k\times d_v}$, its per-token computation is
\begin{align}
T_t&=\alpha_t S_{t-1}, & r_t&=T_t^\top k_t, \nonumber\\
c_t&=v_t-r_t, & u_t&=\beta_t c_t, \label{eq:gdn}\\
S_t&=T_t+k_tu_t^\top, & o_t&=S_t^\top q_t. \nonumber
\end{align}
The query includes the model's fixed attention scale. $T_t$ is the retained state, $r_t$ is its readout at key $k_t$, and $u_t$ is the gated difference between the incoming value and that readout. One endpoint ordering constructs a finite map $D_{\mathrm{GDN}}^{01}$ for the complete memory recurrence through
\begin{align}
\D T_t&=\alpha_{t,1}\D S_{t-1}+S_{t-1,0}\D\alpha_t, \nonumber\\
\D r_t&=(\D T_t)^\top k_{t,1}+T_{t,0}^\top\D k_t, \nonumber\\
\D u_t&=\beta_{t,1}(\D v_t-\D r_t)+c_{t,0}\D\beta_t, \label{eq:gdn-finite}\\
\D S_t&=\D T_t+k_{t,1}\D u_t^\top+\D k_t\,u_{t,0}^\top, \nonumber\\
\D o_t&=(\D S_t)^\top q_{t,1}+S_{t,0}^\top\D q_t. \nonumber
\end{align}
Swapping the reference and original activations used as local slopes gives $D_{\mathrm{GDN}}^{10}$. Within each GDN layer, we average these memory-recurrence maps and propagate the readout slopes as
\begin{equation}
\bar D_{\mathrm{GDN}}=\frac{D_{\mathrm{GDN}}^{01}+D_{\mathrm{GDN}}^{10}}{2},
\qquad m_{\mathrm{in}}=\bar D_{\mathrm{GDN}}^\top m_o.
\label{eq:gdn-symmetric}
\end{equation}
Here $m_o$ collects the slopes of all readouts $o_t$ in the layer, and $m_{\mathrm{in}}$ collects those returned to the recurrence inputs. Equivalently, apply each transpose map to the same $m_o$ and average the returned slopes. Both maps express the same finite output difference, so their mean preserves it. The recurrence returns slopes for values, keys, queries, and both gates, and propagates state slopes to earlier token positions. Scalar gate nonlinearities and query/key normalization use the surrounding network's finite rules. Output gating uses the symmetric product in Table~\ref{tab:rules}, with $n_t$ the normalized readout and $s_t$ its activated output-gate vector.

\subsection{Product and normalization identities}
For a bilinear map $B$, both
\begin{align}
\D B(a,b)&=B(\D a,\bar b)+B(\bar a,\D b),\\
\D B(a,b)&=B(\D a,b_1)+B(a_0,\D b)
\end{align}
follow by expansion. Equation~\ref{eq:pv} and the SwiGLU rule use the first, symmetric identity. The second identity, with the displayed argument ordering, gives the products in Equation~\ref{eq:gdn-finite}. Its swapped ordering gives the other layer map averaged in Equation~\ref{eq:gdn-symmetric}. Subtraction and state addition are linear.

For RMS normalization $f(x)=w\odot x/r(x)$, $r(x)=\sqrt{\|x\|^2/d+\epsilon}$, the symmetric product rule gives an analytic local map. With $\bar x=(x_0+x_1)/2$, $r_b=r(x_b)$, and $v=w\odot m$, its transpose action is
\begin{equation}
D_f^\top m=
\frac{r_0^{-1}+r_1^{-1}}{2}\,v
-\frac{2\bar x\,\ip{\bar x}{v}}{d\,r_0r_1(r_0+r_1)}.
\end{equation}
The identity follows from $\D r=(2\ip{\bar x}{\D x}/d)/(r_0+r_1)$ and $\D(r^{-1})=-\D r/(r_0r_1)$. The same construction covers fixed learned scale parameters and the query/key normalization used in the memory path.

\section{Implementation and Storage Accounting}
\label{app:computation}

\paragraph{Attention.}
Native endpoint queries, keys, values, and log-normalizers determine the probabilities within each FlashAttention tile. The first sweep accumulates the logarithmic-mean normalizers and weighted row averages in Equation~\ref{eq:softmax-vjp}. Subsequent sweeps accumulate query, key, and value slopes. Endpoint averages reside in linear-size buffers or are computed inside the tiles.

\paragraph{Memory.}
The implementation processes 64-token chunks using FLA's native chunk-state adjoints and batched contractions with saved endpoint states. An associative affine scan computes the finite decay contribution. Convolutions and linear projections propagate slopes through their transpose operations.

\paragraph{Storage.}
For sequence length $T$, head dimension $d$, and tile or chunk width $C$, attention vectors and row statistics occupy $O(Td)$ storage per head. Pairwise products reside in bounded tiles. Dense attention arithmetic is $O(T^2d)$. Chunked memory occupies $O(Td+TC+(T/C)d^2)$ auxiliary storage per head for token slopes, within-chunk products, and boundary states. Both storage expressions are linear in $T$ at fixed $d$ and $C$. Layer replay supplies intermediate activations, with checkpoint storage accounted for separately.

\clearpage
\section{Evaluation Specification}
\label{app:evaluation}
Both models use the released prompts, stored responses, eligible source positions, and evidence annotations. Table~\ref{tab:qwen35-quality} specifies the Qwen3.5 examples for each task. An eligible position is an input position included in the task's attribution and deletion domain. For NIAH, this domain follows the released token selection. VT and HotpotQA recovery select from the passage body.

\paragraph{Baseline comparisons.}
VT and HotpotQA recovery compare all seven methods on the same 448 inputs. The paired development comparison uses the same inputs and stored responses for DT and FT on each model, with one FT recursive step for faithfulness and three for recovery.

\begin{table}[!htb]
\centering\small
\caption{Attribution methods.}
\label{tab:baseline-methods}
\setlength{\tabcolsep}{4pt}
\renewcommand{\arraystretch}{1.15}
\begin{tabular}{@{}>{\raggedright\arraybackslash}p{.17\linewidth}>{\raggedright\arraybackslash}p{.27\linewidth}>{\raggedright\arraybackslash}p{.20\linewidth}>{\raggedright\arraybackslash}p{\dimexpr.36\linewidth-6\tabcolsep\relax}@{}}
\toprule
Method & Principle & Strength & DT advantage \\
\midrule
Perturbation\newline\citep{pan2026flashtrace} & Measure effects of separate masks & Direct intervention & Additive allocation of the joint effect \\
\addlinespace[4pt]
REAGENT\newline\citep{zhao2024reagent} & Iteratively replace tokens and update importance & Black-box attribution & Signed scores without iterative sampling \\
\addlinespace[4pt]
CLP\newline\citep{cifka2023context} & Measure prediction changes as context expands & Reveals newly added information & Allocates the joint effect of selected inputs \\
\addlinespace[4pt]
IFR\newline\citep{ferrando2024information} & Trace contributions to internal representations & Identifies information paths & Includes changes in attention weights \\
\addlinespace[4pt]
AttnLRP\newline\citep{achtibat2024attnlrp} & Propagate output relevance backward & Input and neuron attribution & Explains the measured finite score change \\
\addlinespace[4pt]
FlashTrace\newline\citep{pan2026flashtrace} & Aggregate spans and recursively trace sources & Efficient multi-token attribution & Signed, additive response-score contributions \\
\bottomrule
\end{tabular}
\end{table}

\paragraph{Deletion scores.}
For token attribution scores $A$, DT ranks tokens by $A_i$ for RISE and by $\max(A_i,0)$ for MAS. The evaluator records 21 response scores across 20 deletion steps. Each score sums log-probabilities over the full fixed response plus EOS. Masking all eligible tokens with EOS gives the attribution reference. Full-input and fully-masked scores define the normalization endpoints. The normalized values are clipped to $[0,1]$ and made nonincreasing by a cumulative minimum. For MAS, the remaining attribution share is the positive-score sum over unmasked tokens divided by the initial positive-score sum over eligible tokens. We add its absolute difference from the normalized response score to that score, clip to $[0,1]$, and rescale the resulting curve to $[0,1]$ by its minimum and maximum. MAS is the trapezoidal area under this curve. The case-study curves use the positive ranking and share these endpoints between DT and FT for each model and example.

\paragraph{Recovery targets and budgets.}
For NIAH, DT attributes the full response plus EOS. We select the top 10\% of eligible input tokens ranked by positive-part attribution scores. VT scores the answer tokens and ranks positive body-token scores. Its H2/H4/H6/H10 budgets are 10\%/10\%/20\%/30\%. HotpotQA scores the full response, ranks passage sentences by the sum of their positive token contributions, and selects the longest ranked sentence prefix within 30\% of all body tokens. NIAH/VT report the fraction of annotated evidence tokens recovered. HotpotQA reports Full-support Recall@30\%. Each example scores one if every official supporting fact is selected and zero otherwise. The reported value is the mean across examples.

\begin{table}[htbp]
\centering\small
\caption{Paired development results on the first eight MQ-Q2 and eight MoreHopQA examples per model. This gives 32 model--example pairs. FT uses one recursive hop for RISE/MAS and three for recovery. RISE is signed, while MAS and recovery use positive scores.}
\label{tab:development}
\begin{tabular}{@{}lllrrr@{}}
\toprule
Model & Task & Method & RISE $\downarrow$ & MAS $\downarrow$ & Recovery@10\% $\uparrow$ \\
\midrule
Qwen3-8B & MQ-Q2 & FlashTrace & 0.0685 & 0.1547 & 56.45 \\
 &  & \dt{} & 0.0574 & 0.0806 & 77.52 \\
 & MoreHopQA & FlashTrace & 0.1111 & 0.1969 & --- \\
 &  & \dt{} & 0.0585 & 0.1173 & --- \\
\midrule
Qwen3.5-9B & MQ-Q2 & FlashTrace & 0.0942 & 0.2485 & 63.92 \\
 &  & \dt{} & 0.1115 & 0.1518 & 73.35 \\
 & MoreHopQA & FlashTrace & 0.2356 & 0.2797 & --- \\
 &  & \dt{} & 0.1688 & 0.2426 & --- \\
\bottomrule
\end{tabular}
\end{table}

\section{Implementation Cost Measurements}
\label{app:cost}
DT and one-hop FT use the same inputs and responses in each timing comparison. Time covers the complete attribution call, including returning the source scores. For DT, this includes both forward executions and the reverse pass. Peak allocated device memory includes model weights.

\paragraph{Paired Qwen3 cost.}
Each method has one warm call followed by three measured calls in rotated order on each of 16 development examples. Aggregate times sum the per-example medians.

\paragraph{Response-length comparison.}
The first-page timing panel measures Qwen3-8B on one MetaX C550. Its horizontal axis counts response tokens. Each point is the mean of three complete attribution calls after two warm calls, with a fresh process for each method and length. Error bars show the minimum and maximum measured times.

\clearpage
\section{Detailed Signed Evidence Views}
\begin{figure}[!ht]
\centering
\includegraphics[width=\linewidth]{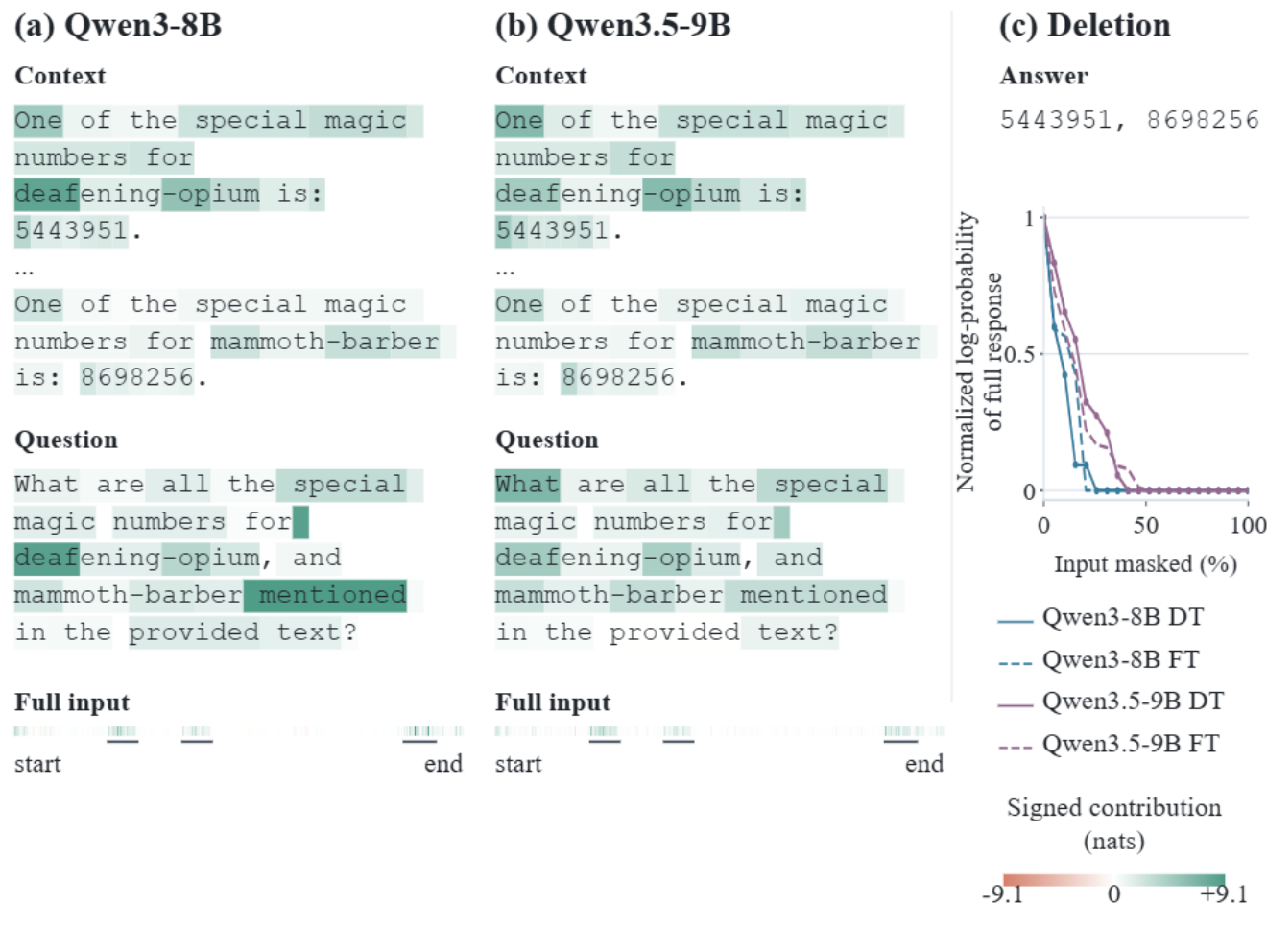}
\caption{\textbf{Signed evidence for retrieving two numbers.} (a,b) The same \texttt{niah\_mq\_q2} input on Qwen3-8B and Qwen3.5-9B. Teal and coral show positive and negative DT contributions in nats. Marks locate excerpts. (c) Each model's DT (solid) and one-hop FT (dashed) curves show the normalized log-probability of the full fixed response plus EOS under 20-step deletion. Color identifies the model. The released normalization uses that model's full-input and fully-masked scores, clipping and a cumulative minimum. DT deletion uses positive ranking. Answer summarizes the full stored response.}
\label{fig:retrieval}
\end{figure}

\begin{figure}[!ht]
\centering
\includegraphics[width=\linewidth]{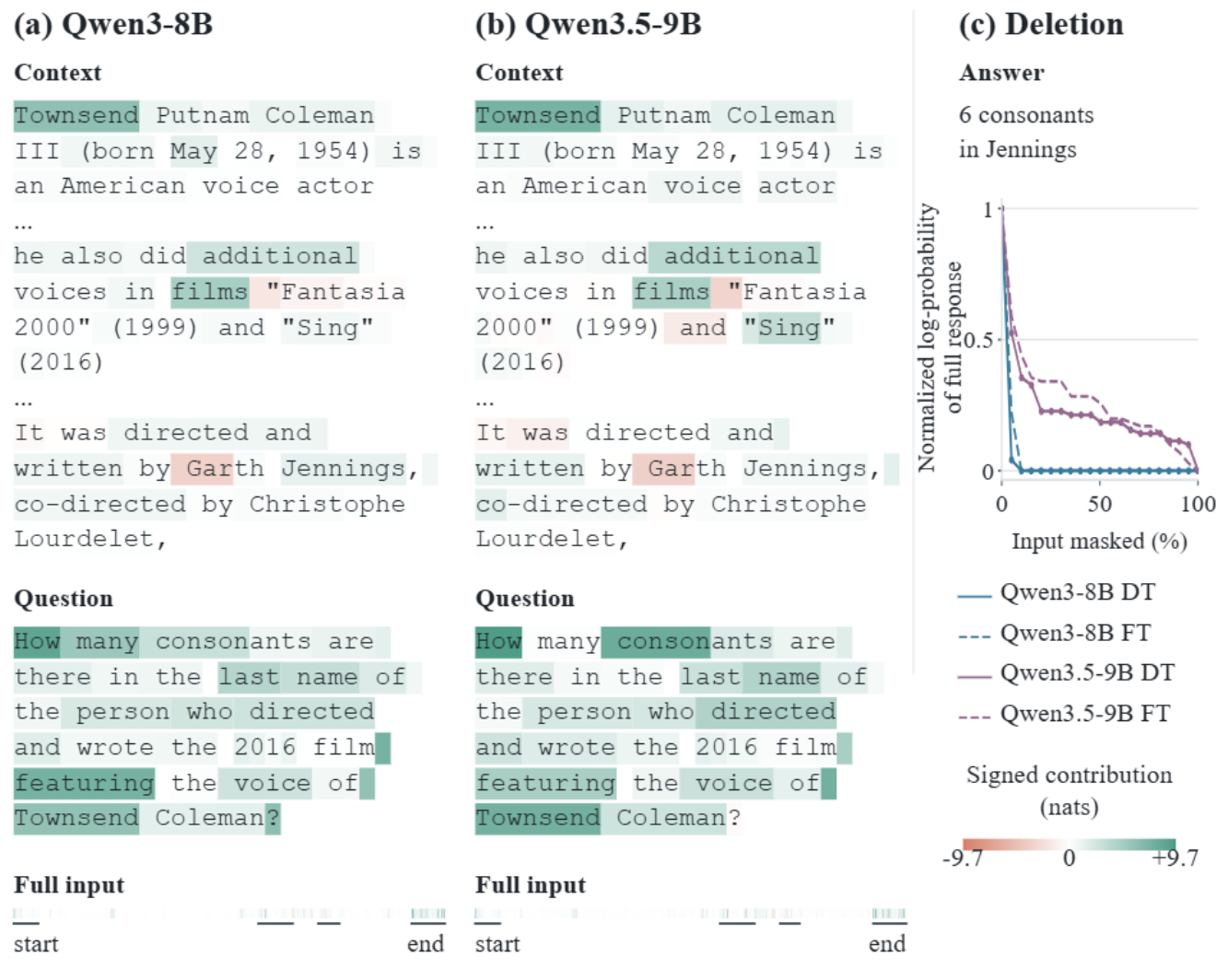}
\caption{\textbf{Signed evidence across a multi-hop question.} (a,b) The MoreHopQA input links Townsend Coleman to \emph{Sing}, Garth Jennings, and his surname's consonants. (c) Normalized log-probability of the full fixed response plus EOS, comparing each model's DT with its own one-hop FT baseline. Model colors, method line styles, and per-model normalization follow Figure~\ref{fig:retrieval}. Both heatmaps share the signed scale within this example.}
\label{fig:multihop}
\end{figure}

\clearpage
\section{Counterfactual Credit and Policy Learning}
\label{app:credit}

Under a fixed continuation policy, the expected return difference between an action and a reference action drawn from that policy equals the action's advantage. This identity expresses the action-independent-baseline principle \citep{sutton1999policy} in counterfactual form. COMA uses the related construction to assign credit by averaging one agent's action while fixing the others \citep{foerster2018counterfactual}.

Let $h$ be a decision history and $\pi$ a fixed continuation policy. Write $G^\pi(h,a;\xi)$ for the integrable return obtained by setting the current action to $a$ and subsequently following $\pi$, with $\xi$ supplying environment and policy randomness. Define
\begin{equation}
Q^\pi(h,a)=\mathbb E_\xi[G^\pi(h,a;\xi)],\qquad
V^\pi(h)=\mathbb E_{a'\sim\pi(\cdot\mid h)}[Q^\pi(h,a')].
\end{equation}
The reference action $a'$ is sampled independently of the selected action $a$, conditional on $h$. The two rollouts may share randomness, with each retaining its correct marginal distribution under the corresponding action intervention.

\begin{proposition}[Exactly averaged counterfactual credit]
Under these definitions, the policy-averaged counterfactual credit satisfies
\begin{equation}
C^\pi(h,a)
=\mathbb E_{a',\xi}\!\left[G^\pi(h,a;\xi)-G^\pi(h,a';\xi)\right]
=Q^\pi(h,a)-V^\pi(h).
\label{eq:credit-advantage}
\end{equation}
For a differentiable policy $\pi_\theta$ with positive probabilities over a fixed finite action set, the score-function learning signals obey
\begin{equation}
\mathbb E_{a\sim\pi_\theta}\!\left[\nabla_\theta\log\pi_\theta(a\mid h)\,C^{\pi_\theta}(h,a)\right]
=\mathbb E_{a\sim\pi_\theta}\!\left[\nabla_\theta\log\pi_\theta(a\mid h)\,Q^{\pi_\theta}(h,a)\right].
\label{eq:credit-gradient}
\end{equation}
\end{proposition}
\begin{proof}
Linearity of expectation gives Equation~\ref{eq:credit-advantage}. The baseline term in Equation~\ref{eq:credit-gradient} is
\begin{equation}
V^{\pi_\theta}(h)\sum_a\pi_\theta(a\mid h)\nabla_\theta\log\pi_\theta(a\mid h)
=V^{\pi_\theta}(h)\nabla_\theta\sum_a\pi_\theta(a\mid h)=0.
\end{equation}
Subtracting this term proves the result. Credit and action value serve as weights in the score-function update.
\end{proof}

This equality provides a concrete target for extending finite attribution to learning. The averaged action credit of a reward-level construction should recover the advantage.

\end{document}